\documentclass[graybox]{svmult}

\usepackage{amsmath}

\usepackage{amsfonts}
\usepackage{graphicx}
\usepackage{subfig}
\usepackage{microtype}
\usepackage{algorithm}
\usepackage{algpseudocode}

\spnewtheorem*{notation}{Notation}{\itshape}{\rmfamily}

\usepackage{mathptmx}       
\usepackage{helvet}         
\usepackage{courier}        
\usepackage{type1cm}        
\usepackage{makeidx}         
\usepackage{graphicx}        
\usepackage{multicol}        
\usepackage[bottom]{footmisc}

\makeindex             

\begin{document}

\title*{Nonparametric Hierarchical Clustering of Functional Data}
\author{Marc Boull\'e, Romain Guigour\`es and Fabrice Rossi}
\institute{Boull\'e, Guigour\`es \at Orange Labs \at 2 av. Pierre Marzin \at 22300 Lannion - France \at \email{marc.boulle@orange.com , romain.guigoures@orange.com}
\and Guigour\`es, Rossi  \at SAMM EA 4543 - Universit\'e Paris 1 \at 90 rue de Tolbiac \at 75013 Paris - France \at \email{romain.guigoures@malix.univ-paris1.fr , fabrice.rossi@univ-paris1.fr}}
%
%
\maketitle

\abstract{
 In this paper, we deal with the problem of curves clustering. We propose a nonparametric method which partitions the curves into clusters and discretizes the dimensions of the curve points into intervals. The cross-product of these partitions forms a data-grid which is obtained using a Bayesian model selection approach while making no assumptions regarding the curves. 
Finally, a post-processing technique, aiming at reducing the number of clusters in order to improve the interpretability of the clustering, is proposed. It consists in optimally merging the clusters step by step, which corresponds to an agglomerative hierarchical classification whose dissimilarity measure is the variation of the criterion. Interestingly this measure is none other than the sum of the Kullback-Leibler divergences between clusters distributions before and after the merges.
The practical interest of the approach for functional data exploratory analysis is presented and compared with an alternative approach on an artificial and a real world data set.
}
\newpage
\section{Introduction}

In functional data analysis (FDA \cite{RamsayEtAl05}), observations are
functions (or curves). Each function is sampled at possibly different
evaluation points, leading to variable-length sets of pairs (evaluation point,
function value).
Functional data arise in many domains, such as daily records of precipitation at a weather station or hardware monitoring where each curve is a time series related to a physical quantity recorded at a specified sampling rate.

Exploratory analysis methods for large functional data sets are needed in
practical applications such as e.g. electric consumption monitoring
\cite{HebrailEtAl10}. They reduce data complexity by combining clustering
techniques with function approximation methods, representing a functional
data set by a small set of piecewise constant prototypes. In this type of
approach, both the number of prototypes and the number of segments (constant
parts of the prototypes) are under user control. On a positive side, this
limits the risk of cognitive overwhelming as the user can ask for a low
complexity representation. Unfortunately, this can also induce
under/over-fitting of the model to the data; additionally the number of
prototypes and the number of segments both need to be tuned, while they can be
adjusted independently in \cite{HebrailEtAl10}, increasing the risk of
over/under-fitting. Other parametric approaches for function clustering and/or
function approximation can be found in
e.g. \cite{CadezEtAl00,chamroukhi_et_al_neurocomputing2010,GaffneySmythNips04,RamsayEtAl05}. All
those methods make (sometimes implicit) assumptions on the distribution of the
functions and/or on the measurement noise.

Nonparametric functional approaches (e.g. \cite{FerratyEtAl06}) have been proposed, in particular in
\cite{GasserEtAl98,DelaigleEtAl10}, where the problem of density estimation of
a random function is considered. However, those models do not tackle directly
the summarizing problem outlined in \cite{HebrailEtAl10} and recalled above. Nonparametric Bayesian approaches based on Dirichlet process have also been applied to the problem of curves clustering. They aim at inferring a clustering distribution on an infinite mixture model \cite{nguyen,Teh2010a}. The clustering model is obtained by sampling the posterior distribution using Bayesian inference methods.

The present paper proposes a new nonparametric exploratory method for
functional data, based on data grid models \cite{BoulleHOPR10}. The method
makes assumption neither on the functional data distribution nor on the
measurement noise. Given a set of sampled functions defined on a common
interval $[a,b]$, with values in $[u,v]$, the method outputs a clustering of
the functions associated to partitions of $[a,b]$ and $[u,v]$ in sub-intervals
which can be used to summarize the values taken by the functions in each
cluster, leading to results comparable to those of \cite{HebrailEtAl10}. Both approaches are for that matter compared in this article.

The method has no parameters and obtains in a fully automated way an optimal
summary of the functional data set, using a Bayesian approach with data
dependent priors. In some cases, especially for large scale data sets, the
optimal number of clusters and of sub-intervals may be too large for a user to interpret all the discovered fine grained patterns in a reasonable time. Therefore, the method is
complemented with a post-processing step which offers the user a way to decrease
the number of clusters in a greedy optimal way. The number of sub-intervals,
that is the level of details kept in the functions, is automatically
adjusted in an optimal way when the number of clusters is reduced. 

The post-processing technique consists in merging successively the clusters in
the least costly way, from the finest clustering model to one single cluster
containing all the curves. It appears that the cost of the merge of two
clusters is a weighted sum of Kullback-Leibler divergences from the merged
clusters to the created cluster which can be interpreted as a dissimilarity
measure between the two clusters that have been merged. Thus, the
post-processing technique can be considered as an agglomerative hierarchical
clustering \cite{HastieEtAl01}. Decision-making tools can be plotted using a
dendrogram and a Pareto chart of the criterion value as a function of the
number of clusters.
 
The rest of the paper is organized as follows. Section~\ref{sec:FunctionApproximationBasedMethods} introduces the problem of curves clustering and relates our method to alternative approaches. Next, in Section~\ref{sec:MODL}, the clustering method based on joint density estimation is introduced. Then, the post-processing technique is detailed in section~\ref{sec:AHC}. In Section~\ref{sec:results} the results of experimentations on an artificial data set and on a power consumption data set are shown. Finally Section~\ref{sec:Conclusion} gives a summary.

%
%
%
%
%
%
\section{Functional data exploratory analysis}
\label{sec:FunctionApproximationBasedMethods}
In this section, we describe in formal terms the data under analysis and the
goals of the analysis.

Let $\mathcal{C}$ be a collection of $n$ functions or curves, $c_i, 1 \leq i
\leq n$, defined from $[a,b]$ to $[u,v]$, two real intervals.  Each curve is
sampled at $m_i$ values in $[a,b]$, leading to a series of observations
denoted $c_i = (x_{ij}, y_{ij})_{j=1}^{m_i}$, with $y_{ij}=c_i(x_{ij})$. 

As in all data exploratory settings, our main goal is to reduce the complexity
of the data set and to discover patterns in the data. We are therefore
interested in finding clusters of similar functions as well as in finding
functional patterns, that is systematic and simple regular shapes in
individual functions. In
\cite{chamroukhi_et_al_neurocomputing2010,HebrailEtAl10} functional patterns
are simple functions such as interval indicator functions or polynomial
functions of low degree: a function is approximated by a linear combination of
such simple functions in \cite{HebrailEtAl10} or generated by a logistic
switching process based on low degree polynomial functions in
\cite{chamroukhi_et_al_neurocomputing2010}. B-splines could also be used as
in \cite{AbrahamEtAl2003} but with no simplification virtues.

Let us denote $k_C$ the number of curve clusters. Given $k_C$ classes
$\mathcal{F}_k$ of ``simple functions'' used to discover functional patterns
(e.g., piecewise constant functions with $P$ segments), the method proposed in
\cite{HebrailEtAl10} finds a partition $(\mathcal{C}_k)_{k=1}^{k_C}$
of $\mathcal{C}$ and $k_C$ simple functions $(f_k\in
\mathcal{F}_k)_{k=1}^{k_C}$ which aim at minimizing
\begin{equation}
\sum_{k=1}^{k_C}\sum_{c_i\in \mathcal{C}_k}\sum_{j=1}^{m_i}\left(y_{ij}-f_k(x_{ij})\right)^2,
\label{eq:fda:model:based}
\end{equation}
which corresponds to a form of K-means constrained by the choice of the segments, in the functional space
$L^2$. The approach of \cite{chamroukhi_et_al_neurocomputing2010} optimizes a
similar criterion obtained from a maximum likelihood estimation of the
parameters of the functional generative model. 

Given a specific choice of the simple function classes, the functional
prototypes $(f_k)_{k=1}^{k_C}$ obtained by
\cite{chamroukhi_et_al_neurocomputing2010,HebrailEtAl10} induce $k_C$
partitions of $[a,b]$ into sub-intervals on which functions are roughly
constant. Those partitions are the main tool used by the analyst to understand
the functional pattern inside each cluster. The general abstract goal of
functional data exploration is therefore to build clusters of similar
functions associated to sub-intervals of the input space of the functions which
summarize the behavior of the functions. 

Bayesian Approaches, as described in \cite{nguyen}, assume that the collection of curves realizations can be represented by a set of canonical curves drawn from a Gaussian Process and organized into clusters. The clusters are described using a label function that is a realization of a multinomial distribution with a Dirichlet prior. Whereas parametric models using a fixed and finite number of parameters may suffer from over- or under-fitting, Bayesian nonparametric approaches were proposed to overcome these issues. By using a model with an unbounded complexity, underfitting is mitigated, while the Bayesian approach of computing or approximating the full posterior over parameters lessens over-fitting \cite{Teh2010a}. Finally, the parameters distribution is obtained by sampling the posterior distribution using Bayesian inference methods such as Markov Chain Monte Carlo \cite{Neal} or Variational Inference \cite{Jordan}. Then a post-treatment is required for the choice of the clustering parameters among their distribution.

The Dirichlet Process prior requires two parameters : a concentration parameter and a base distribution. For a concentration parameter $\alpha$ and a data set containing $n$ curves, the expected number of clusters $\bar{k}$ is $\bar{k}=\alpha \log(n)$ \cite{Wallach}. Hence, the concentration parameter has a significant impact on the obtained number of clusters. For that matter, according to \cite{Vogt}, one should not expect to be able to reliably estimate this parameter.

Our method - named MODL and detailed in Section \ref{sec:MODL} - is comparable to approaches based on Dirichlet process (DP) in so much as all estimate a posterior probability based on the likelihood and a prior distribution of the parameters. The methods are also nonparametric with an unbounded complexity, since the number of parameters is not fixed and grows with the amount of available data.

Nevertheless, MODL is intrinsically different from the DP based methods. First, approaches based on DP are Bayesian and yield a distribution of clusterings, the final clustering being selected using a post-treatment like chosing the mode of the posterior distribution or by studying the clusters co-occurence matrix. By contrast, MODL is a MAP approach, the most probable model is directly obtained using optimization algorithms.
Secondly, MODL is not applied on the values but on the order statistics of the sample. One first benefit is to avoid outliers or scaling problems. By using order statistics, the retrieved models are invariant by any monotonic transformation of the input data, which makes sense since the method aims at modeling the correlations between the variables, not the values directly.
Then, DP based methods consider distributions of the parameters that lie in $\mathbb{R}$ or any continuous space, which measure is consequently infinite. As for MODL, the correlations between the variables are modeled on a sample. In the case of curves clustering, these variables are the location $X$, the corresponding curve realization $Y$, and the curve label $C$. This allows to work on a finite discrete space and thus to simplify the model computation, that mainly comes down to counting problems.
Finally, the MODL approach is clearly data dependant. In a first phase, the data sample is used cautiously to build the model space and the prior : only the size of the sample and the values (or empirical ranks) of each variable taken independently are exploited. The correlation model is inferred in a second phase, using a standard MAP approach. Hence, proving the consistency of this data dependant modeling technique is still an open issue. Actually, experimental results with both reliable and fine grained retrieved patterns show the relevancy of the approach.%
%
%
%
%
%
\section{MODL Approach for Functional Data Analysis}
\label{sec:MODL}
In this section, we summarize the principles of data grid models, detailed in \cite{BoulleHOPR10}, and apply this approach on the functional data. 

\subsection{Data Grid Models}

Data grid  models \cite{BoulleHOPR10} have been introduced for the data preparation phase of the data mining process \cite{ChapmanEtAl00}, which is a key phase, both time consuming and critical for the quality of the results.
They allow to automatically, rapidly and reliably evaluate the class conditional probability of any subset of variables in supervised learning and the joint probability in unsupervised learning. 
Data grid models are based on a partitioning of each variable into intervals in the numerical case and into groups of values in the categorical case. The cross-product of the univariate partitions forms a multivariate partition of the representation space into a set of cells. This multivariate partition, called data grid, is a piecewise constant nonparametric estimator of the conditional or joint probability. The best data grid is searched using a Bayesian model selection approach and efficient combinatorial algorithms.

\subsection{Application to Functional Data}

We propose to represent the collection $\mathcal{C}$ of $n$ curves as a unique data set with $m=\sum_{i=1}^n m_i$ observations and three variables, $C$ to store the curve identifier, $X$ and $Y$ for the point coordinates.
We can apply the data grid models in the unsupervised setting to estimate the joint density $p(C, X, Y)$ between the three variable.
The curve variable $C$ is grouped into clusters of curves, whereas each point dimension $X$ and $Y$ is discretized into intervals. The cross-product of these univariate partitions forms a data grid of cells, whith a peacewise constant joint density estimation per triplet of curve cluster, $X$ interval  and $Y$ interval.
As $p(X, Y | C) = \frac {p(C, X, Y)} {p(C)}$, this can also be interpreted as an estimator of the joint density between the point dimensions, which is constant per cluster of curves. This means that similar curves with respect to the joint density of their point dimensions will tend to be grouped into the same clusters. It is noteworthy that the $(X,Y)$ discretization is optimized globally for the set of all curves and not locally per cluster as in \cite{HebrailEtAl10}.

We introduce in Definition~\ref{FunctionalDataClusteringModel} a family of functional data clustering models, based on clusters of curves, intervals for each point dimension, and a multinomial distribution of all the points on the cells of the resulting data grid. 
\begin{definition}
\label{FunctionalDataClusteringModel}
A functional data clustering model is defined by:
\begin{itemize} 
	\item a number of clusters of curves,
	\item a number of intervals for each point dimension,
	\item the repartition of the curves into the clusters of curves,
	\item the distribution of the points of the functional data set on the cells of the data grid,
	\item the distribution of the points belonging to each cluster on the curves of the cluster.\\
\end{itemize}
\end{definition}

\begin{notation}
$\phantom{x}$
\begin{itemize}
  \item $\mathcal{C}$: collection of curves, size $n=|\mathcal{C}|$.
  \item $\mathcal{P}$: point data set containing all points of $\mathcal{C}$ using 3 variables, size $m=|\mathcal{P}|$.
	\item $C$: curve variable
	\item $X, Y$: variables for the point dimensions
	\item $k_C$: number of clusters of curves
	\item $k_X, k_Y$: number of intervals for variables $X$ and $Y$
	\item $k = k_C k_X k_Y$: number of cells of the data grid
	\item $n_{i_C}$: number of curves in cluster $i_C$
	\item $m_i$: number of points for curve $i$
	\item $m_{i_C}$: cumulated number of points for curves of cluster $i_C$
	\item $m_{j_X}$, $m_{j_Y}$: cumulated number of points for intervals $j_X$ of $X$ and $j_Y$ of $Y$
	\item $m_{i_C j_X j_Y}$: cumulated number of points for cell $(i_C, j_X, j_Y)$ of the data grid
\end{itemize}
\end{notation}
We assume that the numbers of curves $n$ and points $m$ are known in advance and we aim at modeling the joint distribution of the $m$ points on the curve and the point dimensions. In order to select the best model, we apply a Bayesian approach, using the prior distribution on the model parameters described in Definition~\ref{FunctionalDataClusteringPrior}.

\begin{definition} 
\label{FunctionalDataClusteringPrior}
The prior for the parameters of a functional data clustering model are chosen hierarchically and uniformly at each level: 
\begin{itemize}
\item	the numbers of clusters $k_C$ and of intervals $k_X, k_Y$ are independent from each other, and uniformly distributed between $1$ and $n$ for the curves, between $1$ and $m$ for the point dimensions,
\item for a given number $k_C$ of clusters, every partitions of the $n$ curves into $k_C$ clusters are equiprobable,
\item for a model of size $(k_C, k_X, k_Y)$, every distributions of the $m$ points on the $k=k_C k_X k_Y$ cells of the data grid are equiprobable,
\item	for a given cluster of curves, every distributions of the points in the cluster on the curves of the cluster are equiprobable,
\item for a given interval of $X$ (resp. $Y$), every distributions of the ranks of the $X$ (resp. $Y$) values of points are equiprobable.
\end{itemize}
\end{definition}

Taking the negative log of the posterior probability of a model given the data, this provides the evaluation criterion given in Theorem \ref{FunctionalDataClusteringTheorem}, which specializes to functional data clustering the unsupervised data grid model general criterion \cite{BoulleHOPR10}.

\begin{theorem}
\label{FunctionalDataClusteringTheorem}
A functional data clustering model $M$ distributed according to a uniform hierarchical prior is Bayes optimal if the value of the following criteria is minimal 
\begin{equation}
\label{eq:crit}
\begin{split}
c(M) =& -log(P(M))-log(P(\mathcal{P}|M)) \\
=& \log n + 2 \log m + \log B(n, k_C)\\
 &+ \log \binom {m + k - 1} {k - 1}
  + \sum_{i_C=1}^{k_C} {\log \binom {m_{i_C} + n_{i_C} - 1} {n_{i_C} - 1}}\\
 &+ \log m! - \sum_{i_C = 1}^{k_C} {\sum_{j_X = 1}^{k_X} {\sum_{j_Y = 1}^{k_Y} {\log m_{i_C j_X j_Y}!}}}\\
 &+ \sum_{i_C = 1}^{k_C} {\log m_{i_C}!} 
  - \sum_{i = 1}^{n} {\log m_i!}\\
 &+ \sum_{j_X = 1}^{k_X} {\log m_{j_X}!}
 + \sum_{j_Y = 1}^{k_Y} {\log m_{j_Y}!}
\end{split}
\end{equation}

\end{theorem}

$B(n,k)$ is the number of divisions of $n$ elements into $k$ subsets (with eventually empty subsets). When $n=k$, $B(n,k)$ is the Bell number. In the general case, $B(n,k)$ can be written as $B(n,k) = \sum_{i=1}^k {S(n,i)}$, where $S(n,i)$ is the Stirling number of the second kind \cite{AbramowitzEtAl70}, which stands for the number of ways of partitioning a set of $n$ elements into $i$ nonempty subsets.

As negative log of probabilities are coding lengths, the model selection technique is similar to a minimum description length approach \cite{Rissanen78}.
The first line in Formula~\ref{eq:crit} relates to the prior distribution of the numbers of cluster $k_C$ and of intervals $k_X$ and $k_Y$, and to the specification of the partition of the curves into  clusters.
The second line represents the specification of the parameters of the multinomial distribution of the $m$ points on the $k$ cells of the data grid, followed by the specification of the multinomial distribution of the points of each cluster on the curves of the cluster.
The third line stands for the likelihood of the distribution of the points on the cells, by the mean of a multinomial term.
The last line corresponds to the likelihood of the distribution of the points of each cluster on the curves of the cluster, followed by the likelihood of the distribution of the ranks of the $X$ values (resp. $Y$ values) in each interval.

\subsection{Optimization Algorithm}

The optimization heuristics have practical scaling properties, with O$(m)$ space complexity and O$(m \sqrt m \log m)$ time complexity.
The main heuristic is a greedy bottom-up heuristic, which starts with a fine grained model, with a few points per interval on $X$ and $Y$ and a few curves per cluster, considers all the merges between clusters and adjacent intervals, and performs the best merge if the criterion decreases after the merge, as detailed in Algorithm~\ref{alg:gbum} \\
This heuristic is enhanced with post-optimization steps (moves of interval bounds and of curves across clusters), and embedded into the variable neighborhood search (VNS) meta-heuristic \cite{HansenEtAl01}, which mainly benefits from multiple runs of the algorithm with different initial random solutions. 

\begin{algorithm}
\caption{Greedy Bottom Up Merge Heuristic}
\label{alg:gbum}
\begin{algorithmic}
	\Require $M$ (initial solution)
	\Ensure $M^*\mbox{ ; } c(M^*) \leq c(M)$
	\State $M^* \gets M$
	\While{$\mbox{solution is improved}$}
		\State $M' \gets M^*$
		\ForAll {merge $u$ between 2 clusters or adjacent intervals of X or Y}
			\State $M^+ \gets M^*+u$
			\If {$c(M^+)<c(M')$}
				\State $M' \gets M^+$
			\EndIf
		\EndFor
		\If {$c(M')<c(M^*)$}
			\State $M^* \gets M'$ (improved solution)
		\EndIf
	\EndWhile
\end{algorithmic}
\end{algorithm}

The optimization algorithms summarized above have been extensively evaluated in \cite{BoulleHOPR10}, using a large variety of artificial data sets, where the true data distribution is known.
Overall, the method is both resilient to noise and able to detect complex fine grained patterns. It is able to approximate any data distribution, provided that there are enough instances in the train data sample.

\section{Agglomerative Hierarchical Clustering}
\label{sec:AHC}
The model carried out by the method detailed in the section~\ref{sec:MODL} is optimal according to the criterion introduced in Theorem~\ref{FunctionalDataClusteringTheorem}.
This parameter-free solution allows to track fine and relevant patterns without over-fitting. 
This provides a suitable initial solution to lead an exploratory analysis.
Still, this initial solution may be too fine for an easy interpretation.
We propose here a post-processing technique which aims at simplifying the clustering while minimizing the loss of information.
This allows to explore the retrieved patterns at any granularity, up to the finest model, without any user parameter.

We first study the impact of a merge on the criterion, then focus on the properties of the proposed dissimilarity measure and finally describe the agglomerative hierarchical clustering heuristic.
It is noteworthy than the same modeling criterion is optimized both for building the initial clustering and for aggregating the clusters in the agglomerative heuristic.

%

\subsection{The Cost of Merging two Clusters}
Let $M_{1_C,2_C}$ and $M_{\gamma_C}$ be two clustering models, the first one is the model before the merge of the clusters $1_C$ and $2_C$, the second one is the model after the merge, that yields a new cluster $\gamma_C=1_C\cup2_C$. We denote $\Delta c(1_C,2_C)$ the cost of the merge of $1_C$ and $2_C$, defined as: 
\begin{eqnarray*}
\Delta c(1_C,2_C) = c(M_{\gamma_C}) - c(M_{1_C,2_C})
\end{eqnarray*}
It results from Theorem~\ref{FunctionalDataClusteringTheorem} that the clustering model $M_{\gamma_C}$ is a less probable MODL explanation of the data set $\mathcal{P}$ than $M_{1_C,2_C}$ according to a factor based on $\Delta c(1_C,2_C)$. 
\begin{eqnarray}
p(M_{\gamma_C}|\mathcal{P})=e^{-\Delta c(1_C,2_C)}p(M_{1_C,2_C}|\mathcal{P})
\label{eq:varcrit1}
\end{eqnarray}
We focus on the asymptotic behavior of $\Delta c(1_C,2_C)$ when the number of data points $m$ tends to infinity.

\begin{theorem}The criterion variation is asymptotically equal to a weighted sum of the Kullback-Leibler divergences from the clusters $1_C$ and $2_C$ to $\gamma_C$, estimated on the $k_X \times k_Y$ bivariate discretization.

\begin{equation}
\label{eq:varcrit2}
\begin{split}
\Delta c(1_C,2_C) =& m_{1_C}D_{KL}(1_C||\gamma_C)+m_{2_C}D_{KL}(2_C||\gamma_C)+O(log(m_{\gamma_C})) 
\end{split}
\end{equation}
\end{theorem}
\begin{proof}
The full proof is left out for brevity. Mainly, the computation of $\Delta c(1_C,2_C)$ makes some prior terms (2 first lines of Formula~\ref{eq:crit}) vanish and bounds the other ones by $O(log(m_{\gamma_C}))$ terms. Then, using the Stirling approximation $log(m!)=m(log(m)-1)+O(log(m))$, the variation of the likelihood (the two last lines of Formula~\ref{eq:crit}) can be rewritten as a weighted sum of Kullback-Leibler divergences.
\end{proof}

\subsection{The Cost of a Merge as a Dissimilarity Measure}

As the criterion defined in Theorem~1 is used to find the best model, we naturally chose it to evaluate the quality of the clustering. When two clusters are merged, the criterion decreases and its resulting variation can be viewed as dissimilarity between both clusters. When the number of points tends to infinity, the dissimilarity measure asymptotically converges to a weighted sum of Kullback-Leibler divergence (see Theorem~2). This divergence is a non symmetric measure of the difference between two distributions \cite{CoverEtAl91}. 
The variation of the criterion $\Delta c$ has some interesting properties. First, it is symmetrical, $\Delta(1_C,2_C)=\Delta(2_C,1_C)$. Then, $\Delta c(1_C,2_C)$ is asymptotically non-negative since the Kullback-Leibler divergence is also \cite{CoverEtAl91}.
The weights have an important impact on the merge in the case of unbalanced clusters. A trade-off is achieved between merging two balanced clusters with similar distributions and merging two different clusters, one of them having a tiny weight. The best merge is the one with the least loss of information, as $c(M)$ can be interpreted as the total coding length of the clustering model plus the data points given the model.

\subsection{The Agglomerative Hierarchical Classification}

The principle of the agglomerative clustering is to merge successively the clusters in order to build a tree called dendrogram. The usual dissimilarity measures for the dendrogram are based on Euclidean distances (Single-Linkage, Complete-Linkage, Ward ...). Here we build a dendrogram using the criterion variation $\Delta c$. Due to the properties of this dissimilarity measure, the resulting dendrogram is well-balanced. Indeed, given the trade-off between merging similarly distributed clusters and merging tiny with large clusters, we obtain clusters with comparable sizes at each level of hierarchy.

Let us notice that during the agglomerative process, the best merge can relate either to the cluster variable $C$ or to the points dimensions $X$ or $Y$. Therefore, the granularity of the representation of the curves coarsens as the number of clusters decreases. As a consequence, the dissimilarity measure between two clusters of a partition ``coarsens'' together with the coarsening of the other partitions. This makes sense since fewer clusters in the partition need a less discriminative similarity measure to be distinguished. 
It is noteworthy that during the agglomerative process, partitions are coarsened but not re-optimized by locally moving the bounds of the intervals. 
Although this may be sub-optimal, this allows to ease the exploratory analysis by using the same family of nested intervals at any model granularity.
%
%
%
%
%
%
\newpage
\section{Experiments}
\label{sec:results}
In this section, we first highlight properties of our approach using an artificial data set and then apply it on a real-life data set, next we successively merge the clusters and finally show what kind of exploratory analysis can be performed.

\subsection{Experiments on an artificial data set}

 A variable $z$ is sampled from an uniform distribution: $Z \sim \mathcal{U}(-1,1)$. $\varepsilon_i$ denotes a white Gaussian noise: $E \sim \mathcal{N}(0,0.25)$. Let us consider the four following distributions:
\begin{itemize}
\item $f_1 : x = z + \varepsilon_x \mbox{ , } y = z + \varepsilon_y $
\item $f_2 : x = z + \varepsilon_x \mbox{ , } y = -z + \varepsilon_y $
\item $f_3 : x = z + \varepsilon_x \mbox{ , } y = \alpha z + \varepsilon_y$ \\ $\phantom{f_3 : }\mbox{ with } \alpha \in \{-1,1\}$ \\ $\phantom{f_3 : } \mbox{ and } p(\alpha=-1)=p(\alpha=1)$
\item $f_4 : x = (0.75 + \varepsilon_x) cos(\pi(1+z)) \mbox{ ,}$ \\ $\phantom{f_3 : :} y = (0.75 + \varepsilon_y) sin(\pi(1+z))$
\end{itemize}

\begin{figure}[!ht]
  \centering
  \subfloat[$f_1$]{\includegraphics[width=0.38\textwidth]{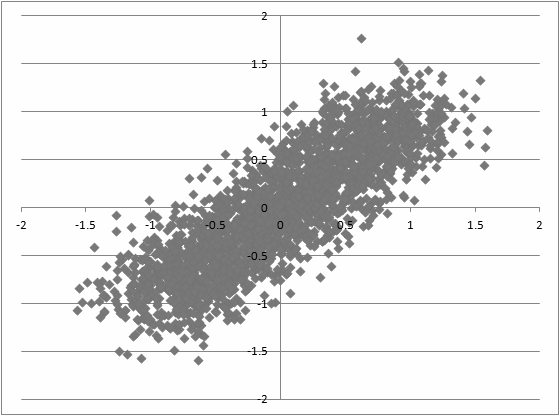}}              
  \subfloat[$f_2$]{\includegraphics[width=0.38\textwidth]{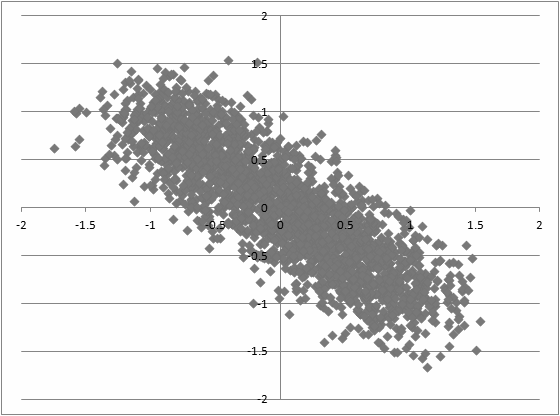}} 
  \\
    \subfloat[$f_3$]{\includegraphics[width=0.38\textwidth]{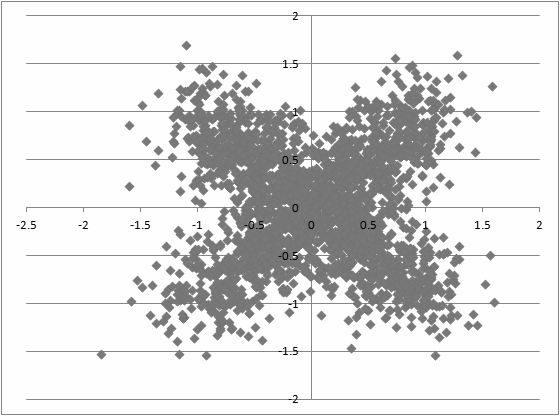}}              
  \subfloat[$f_4$]{\includegraphics[width=0.38\textwidth]{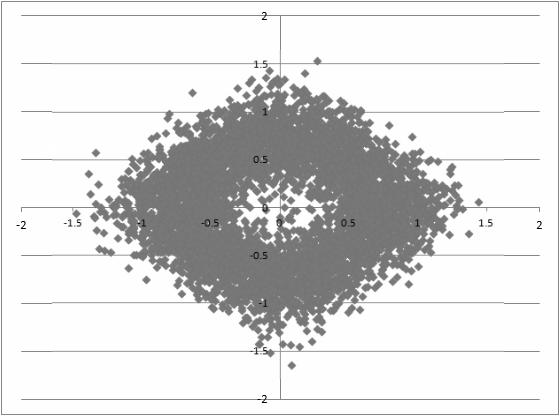}} 
\caption{Artificially generated distributions.}
  \label{Fig:artificial_Pattern}
\end{figure}
We generate a collection of $40$ curves using each distribution defined previously ($10$ curves per distribution). We generate a data set $\mathcal{P}$ of $10^5$ points. Each point is a triple of values with a randomly chosen curve (among 40), a $x$ and a $y$ value generated according to the distribution related to the curve.

We apply our functional data clustering method introduced in Section \ref{sec:MODL} on subsets of $\mathcal{P}$ of increasing sizes. The experiment is running 10 times per subset of points that are resampled each time. The graph on Figure~\ref{Fig:clustersNumber} displays the average number of clusters and the number of X and Y intervals for a given number of points $m$. For very small subsets (below 400 data points), there are not enough data to discover significant patterns, and our method produces one single cluster of curves, with one single interval for the $X$ and $Y$ variables. From 400 data points, the numbers of clusters and intervals start to grow. Finally with only 25 points per curve on average, that is 1000 points in the whole point data set, our method recovers the underlying pattern and produces four clusters of curves related to the $f_1$, $f_2$, $f_3$ and $f_4$ distributions.

Despite the method retrieved the actual number of clusters, below 2000 data points, the clusters may not be totally pure and some curves misplaced into clusters. In our experiments, for $1000$ data points, $2\%$ of the curves are misplaced on average, while for $2000$ points, all the curves are systematically placed in their actual cluster.

It is noteworthy that by growing the size of the subset beyond 2000 data points, the number of retrieved patterns is constant and equal to four. By contrast, the number of intervals grows with the number of data points. This shows the good asymptotic behaviour of the method: it retrieves the true number of patterns and exploits the growing number of data to better approximate the pattern shapes.
\begin{figure}[!ht]
  \centering
  \includegraphics[width=0.8\textwidth]{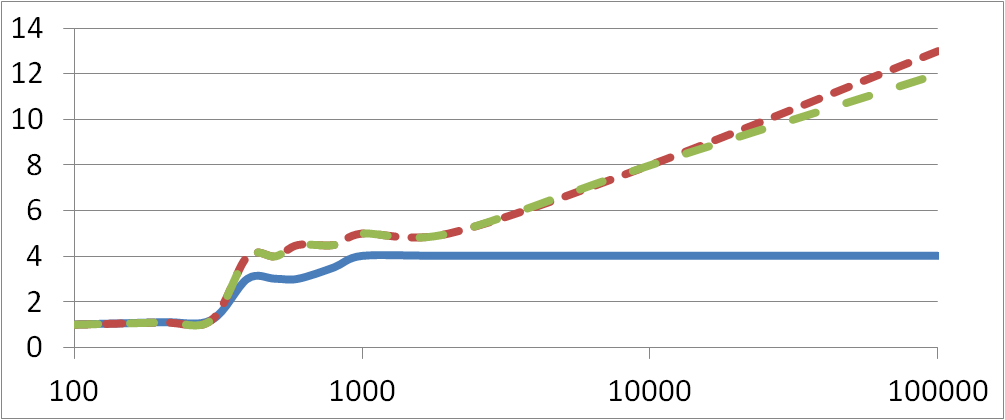}
\caption{Number of clusters (solid line), number of $X$ intervals (tight dotted line) and number of $Y$ intervals (spaced dotted line) for a given number of data points $m$. }
  \label{Fig:clustersNumber}
\end{figure}

Regarding the results of the experiments on this data set, it is noteworthy that MODL does not require the same point locations for each curve. This may be an usefull property to make a clustering of functionnal data for which the measurement have not been recorded at regular intervals. Moreover, beyond the clustering of functional data, our method is able to deal with distributions. Thus, it is possible to detect clusters of multimodal distributions like the ones generated using $f_3$ and $f_4$.

\subsection{Analysis of a power consumption data set}

We use the data set \cite{HebrailEtAl10} which consists in the electric power consumption recorded in a personal home during almost one year ($349$ days). Each curve consists in 144 measurements which give the power consumption of a day at a $10$ minutes sampling rate. There are $50{,}256$ data points and three features: the time of the measure $X$, the power measure $Y$ and the day identifier $C$. The study of this data set aims at grouping the days according to the characteristic of the power consumption of each day. First, the optimal model is computed using the MODL approach. Finally the approach is compared to that of \cite{HebrailEtAl10}. \\

\textit{The MODL-Optimal Discretization.} The optimal clustering consists in a data grid defined by $57$ clusters, $7$ intervals on $X$ and $10$ on $Y$. This means that the $349$ recorded days have been grouped into 57 clusters, each day has been discretized into $7$ time segments and the power measures into $10$ power segments. This result highlights some characteristic days, such as the workdays, the days off or the days when nobody is at home.
The summarized prototypes, represented by piecewise constant lines, show the average power consumption per time segment. The conditional probabilities of the power segments given the time segments are represented by grey cells, where the grey level shows the related conditional probability. The first representation has been chosen in order to simplify the reading of the curve, and the second to highlight some interesting phenomena such as the multimodal distributions of data points within the time segments.\\

\begin{figure}[!ht]
  \centering
  \subfloat[]{\includegraphics[width=0.48\textwidth]{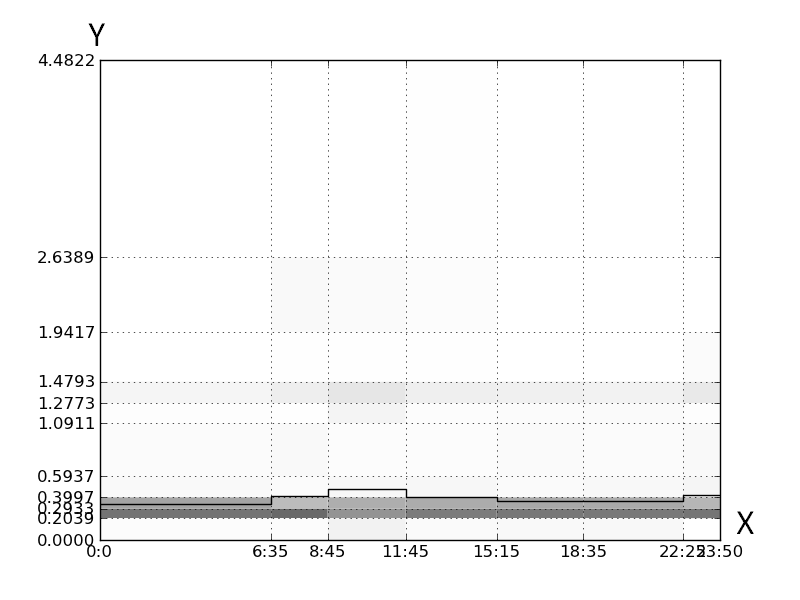}}              
  \subfloat[]{\includegraphics[width=0.48\textwidth]{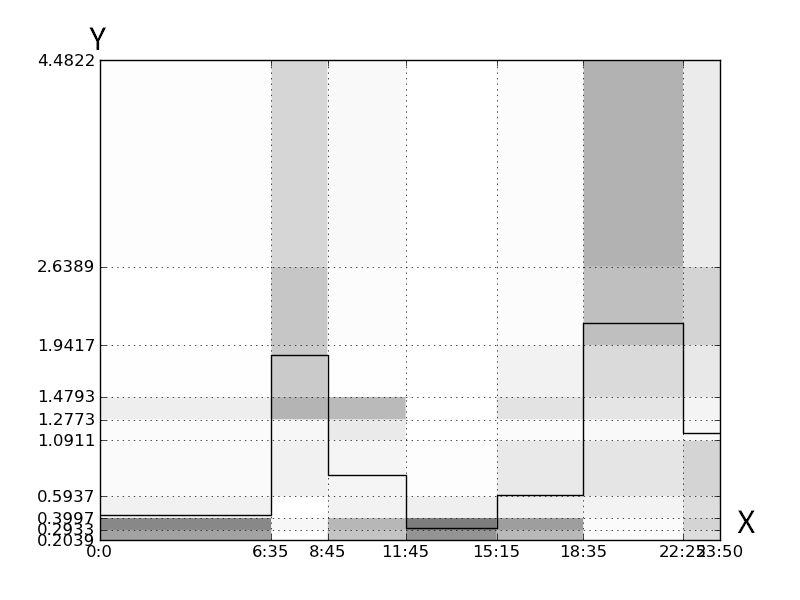}} 
\caption{Two examples among the 57 clusters, the plots display the summarized prototypes and the conditional probabilities represented by darkened cells. Figure~(a) represents the largest cluster, typifying days where the power consumption is very low and almost constant ; the residents were probably not at home. Figure~(b), that is the second largest cluster, shows a workday with a low consumption during the night and the office hours, and with peaks in the morning and evening.}
  \label{Fig:57clusters}
\end{figure}

\textit{Multimodal distributions.} In Figure \ref{Fig:57clusters}.(b), we notice that the prototype is located between two dark cells for the third time segment. This means that the majority of the data points have been recorded in the higher and the lower power segments but rarely in the interval where the prototype is for this time segment. Thus, a multimodal distribution of the data points on this time segment is highlighted, which is confirmed by Figure \ref{Fig:stackcurves}.(b).
Let us notice that \ref{Fig:57clusters}.(a) is another illustration of a multimodal distribution for which the points are more frequent in the lower mode than in the upper one. Overall, the method extends the clustering of curves to clustering of distributions.

\begin{figure}[!ht]
  \centering
  \subfloat[]{\includegraphics[width=0.49\textwidth]{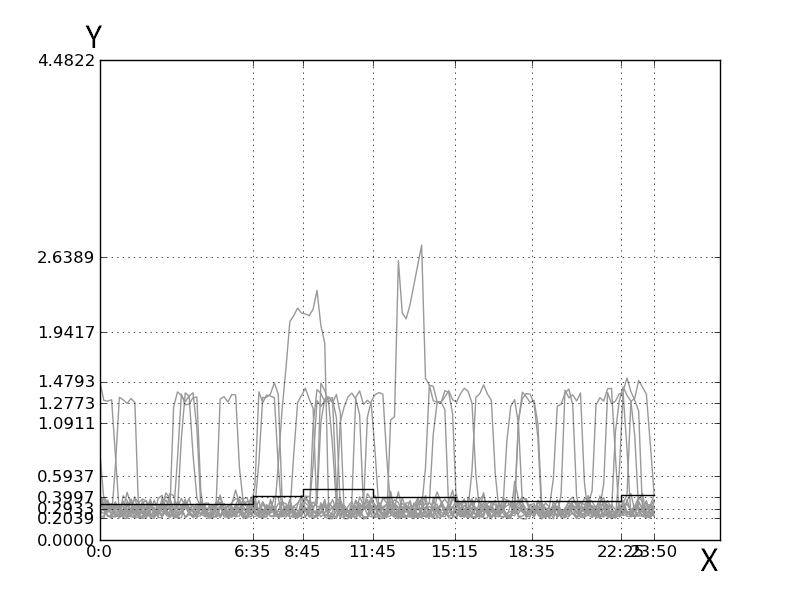}}            
  \subfloat[]{\includegraphics[width=0.49\textwidth]{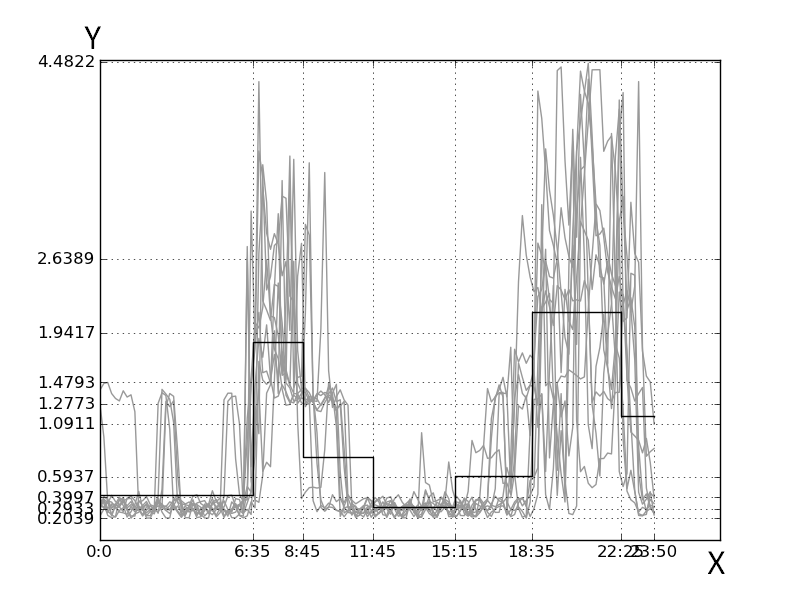}}
  \caption{Prototypes and stacked curves for the clusters of Figures \ref{Fig:57clusters} (a) and (b).}
  \label{Fig:stackcurves}
\end{figure}

\textit{Merging the Clusters.}
Whereas the finest data grid yields a rich clustering and useful information for some characteristic clusters, a more synthetic and easily interpretable view of the power consumption over the year may be desirable in some applications. That is why agglomerative merges have been performed and represented on Figure~\ref{Fig:dendro} by a dendrogram and a Pareto chart presenting the percentage of kept information as a function of the number of clusters. This measure is defined as following:
\begin{definition}
Let $M_{\emptyset}$ be the null model with one cluster of curves and one interval per point dimension, whose data grid consists in one cell containing all the points. Its properties are detailed in \cite{BoulleHOPR10}. We denote $M_{opt}$ the optimal model according to the optimization of the criterion defined in the Theorem~\ref{FunctionalDataClusteringTheorem} and $M_k$ the model resulting from successive merges until obtaining $k$ clusters. The percentage of kept information for $k$ clusters $\tau_k$ is defined as:
\begin{eqnarray*}
\tau_k=\dfrac{c(M_k)-c(M_{\emptyset})}{c(M_{opt})-c(M_{\emptyset})}
\end{eqnarray*}
\end{definition}
The dendrogram is well-balanced and the Pareto chart is concave, which allows to divide by three the number of clusters while keeping almost $90\%$ of the initial information. 

\begin{figure}[!ht]
  \centering
  \subfloat[Dendrogram]{\includegraphics[width=0.49\textwidth]{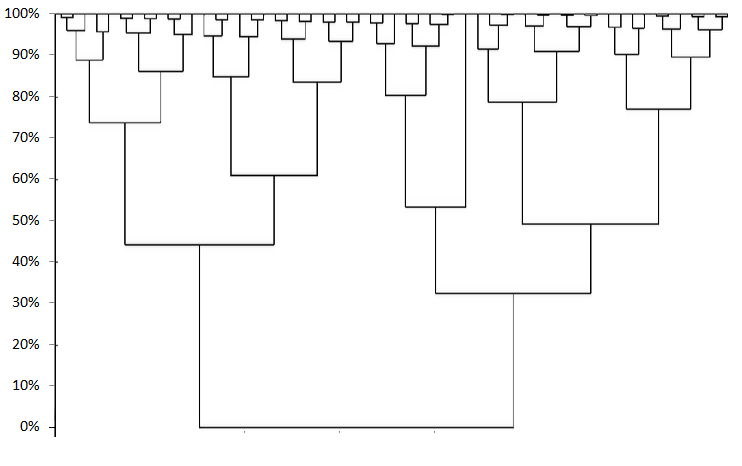}}
  \,
  \subfloat[Pareto chart]{\includegraphics[width=0.49\textwidth]{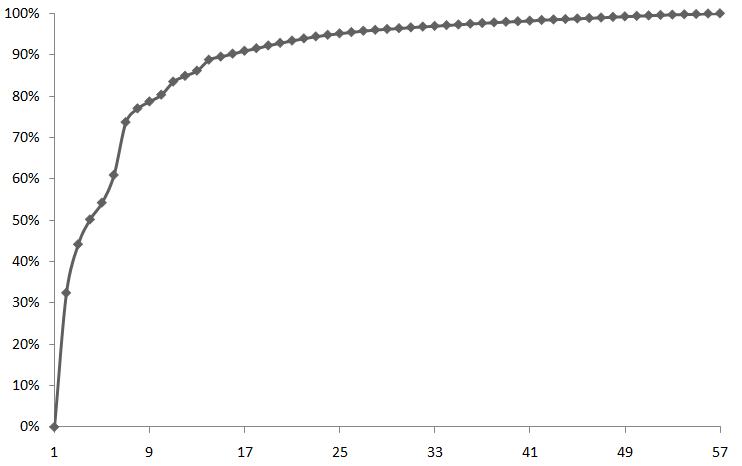}} 
  \caption{Dendrogram and Pareto chart of kept information per number of clusters.}
  \label{Fig:dendro}
\end{figure}

\textit{Comparative analysis of the modeling results.} In order to highlight the differences between the results retrieved using MODL and the approach of \cite{HebrailEtAl10}, we propose to study a simplified data grid by coarsening the optimal model until having four clusters, using the post-processing technique detailed in Section~4. By doing this, $50\%$ of the information is kept and the power consumption and the time discretizations are reduced to four intervals. Contrary to MODL, the approach of \cite{HebrailEtAl10} requires the user to specify the number of clusters and time segments. We applied therefore
their clustering technique with four clusters and a total of sixteen time intervals that are optimally distributed over the four clusters. The clusters retrieved by both approaches are displayed in Figures \ref{fig:MODL4} and \ref{fig:Hebr4}.

\begin{figure}[ht!]
	\centering
		\includegraphics[width=0.75\textwidth]{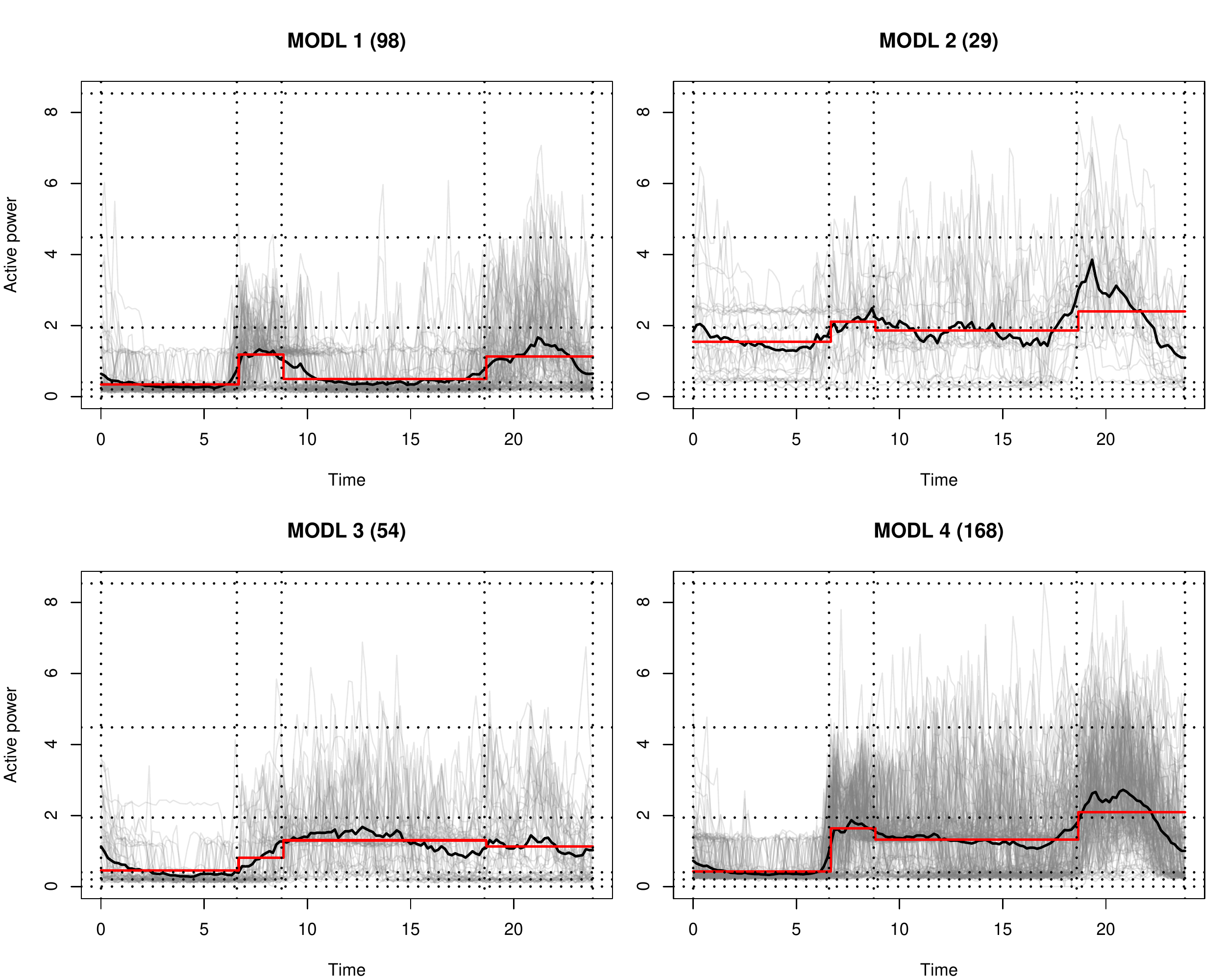}
	\caption{The four clusters of curves retrieved using MODL with the average (black line) and the prototype (red solid line) curves. The number in parenthesis above each curve refers to the number of curves in the cluster.}
	\label{fig:MODL4}
\end{figure}
\begin{figure}[hb!]
	\centering
		\includegraphics[width=0.75\textwidth]{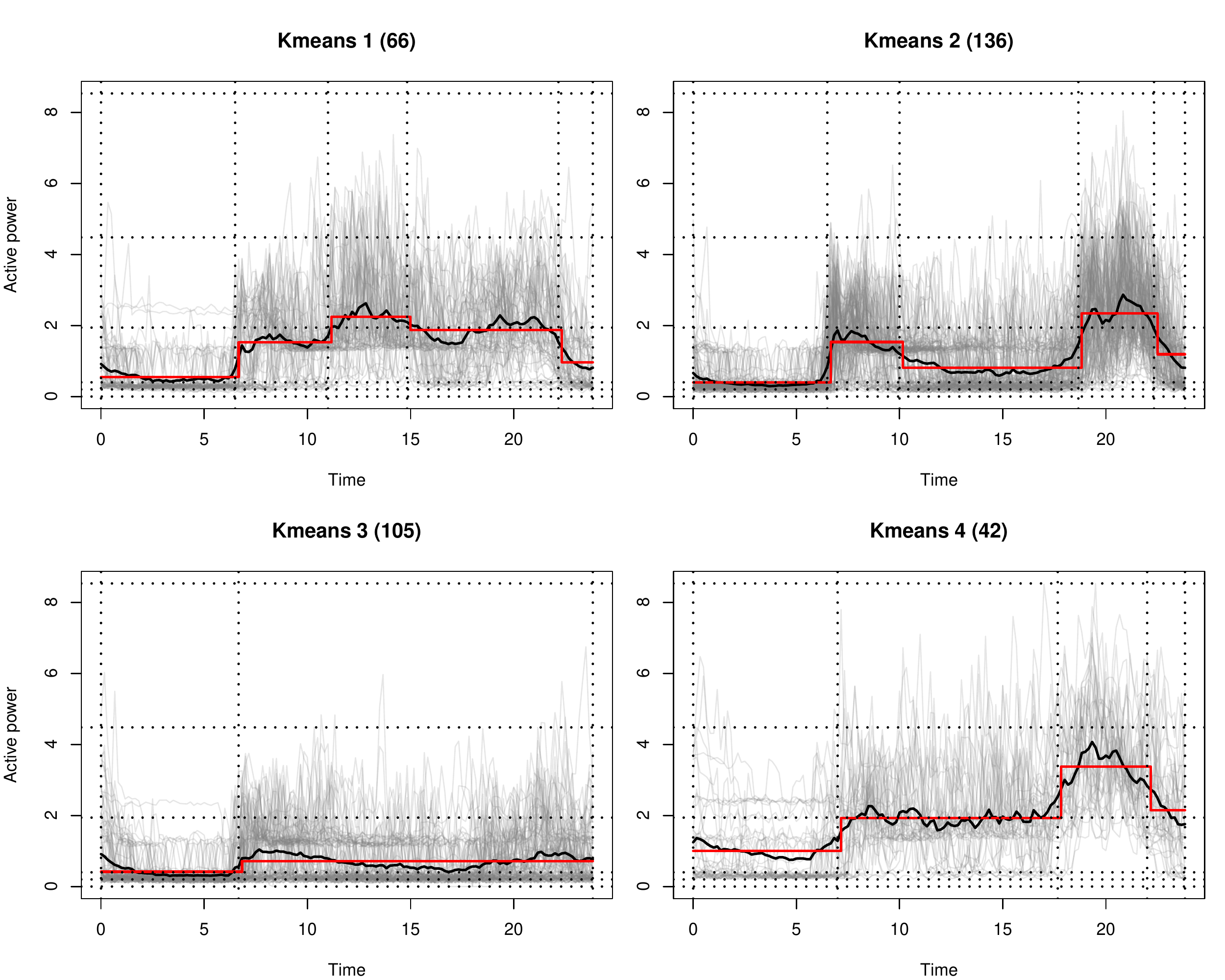}
	\caption{The four clusters of days retrieved using the approach of \cite{HebrailEtAl10} with the average (black line) and the prototype (red solid line) curves. The number in parenthesis above each curve refers to the number of curves in the cluster.}
	\label{fig:Hebr4}
\end{figure}

MODL computes a global discretization for both the time and the power consumption. Conversely, the approach of \cite{HebrailEtAl10} makes a discretization of the temporal variable only, that is different for each cluster of curves. In certain cases like the cluster $3$ of the Figure \ref{fig:Hebr4}, it may be suitable to avoid over-discretizations, and a few number of time segments is better for a local interpretation. However, having common time segments for all the clusters enables an easier comparison between the clusters. In the context of the daily power consumption, MODL enables the identification of four periods: the \textit{night} (midnight - 6.35 AM), the \textit{morning} (6.35 AM - 8.45 AM), the \textit{day} (8.45 AM - 6.35 PM) and the \textit{evening} (6.35 PM - midnight). We are then able to compare the differences in terms of power consumption between the clusters of curves for each period of the day.

The approach of \cite{HebrailEtAl10} is based on the k-means and thus minimizes the variance between the curves locally to each time segment. It is the reason why the prototype are close to the average curves in the clusters obtained by this approach. In MODL, this property is not wanted. As a consequence, the prototype and the average curves seem less correlated. MODL is based on a joint density estimation that yields more complex patterns. To highlight the differences in terms of patterns, we propose to focus on a specific time segment. The first interval (i.e the \textit{night}) found by MODL also exists in the four clusters obtained using the approach of \cite{HebrailEtAl10}. Let us focus on this time segment to investigate on the distributions of the power consumption measurements for each cluster of curves. To do that, we compute the probability density function of the power consumption variable locally to the first time segment, using a kernel density estimator \cite{sheather1991}. The results are displayed in Figures \ref{fig:MODLkde} and \ref{fig:Hebrkde}.

\begin{figure}[ht]
	\centering
		\includegraphics[width=0.75\textwidth]{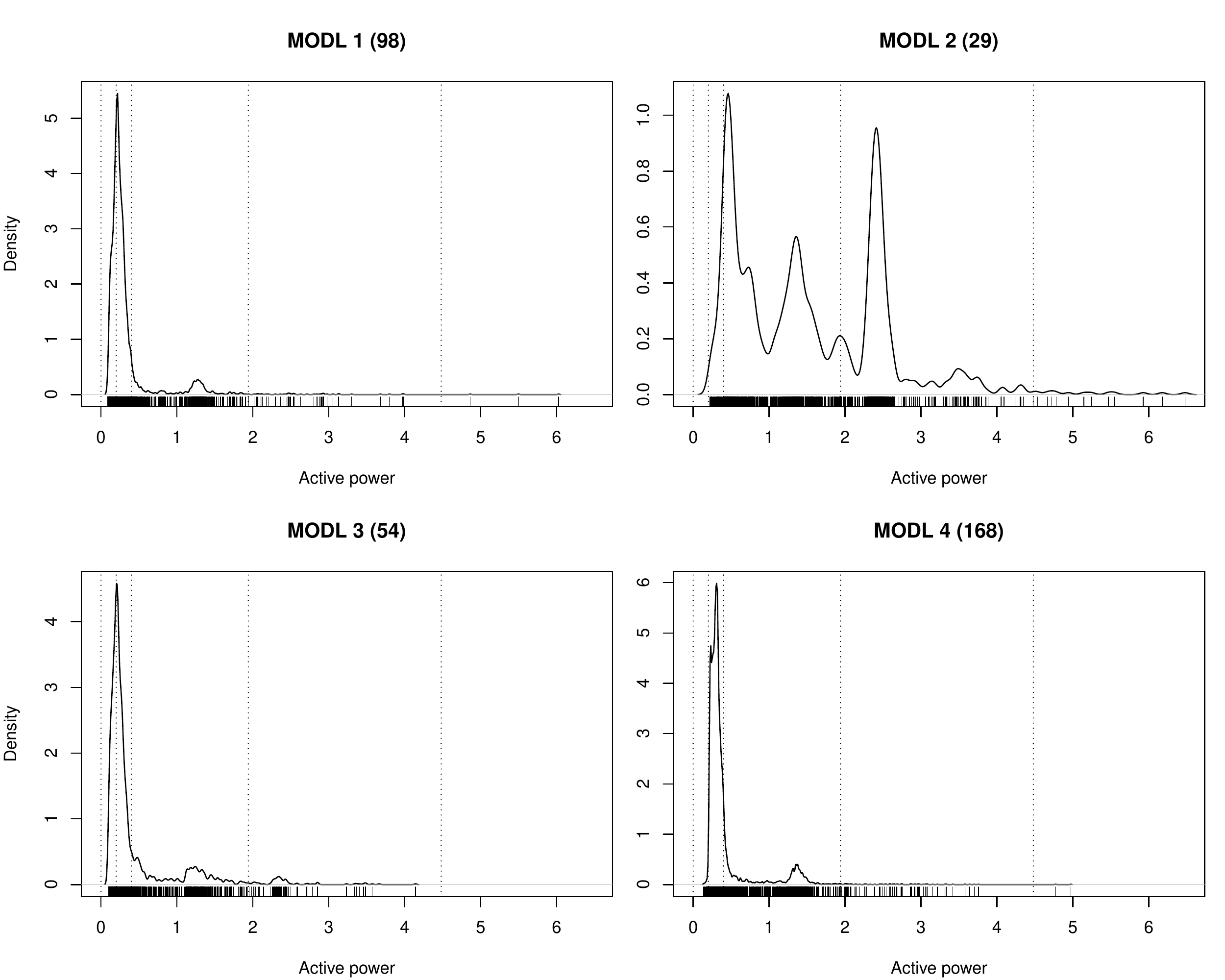}
	\caption{Kernel density estimation of the power consumption measurements between midnight and 6.35 AM for each cluster of curves retrieved using MODL.}
	\label{fig:MODLkde}
\end{figure}
\begin{figure}[ht]
	\centering
		\includegraphics[width=0.75\textwidth]{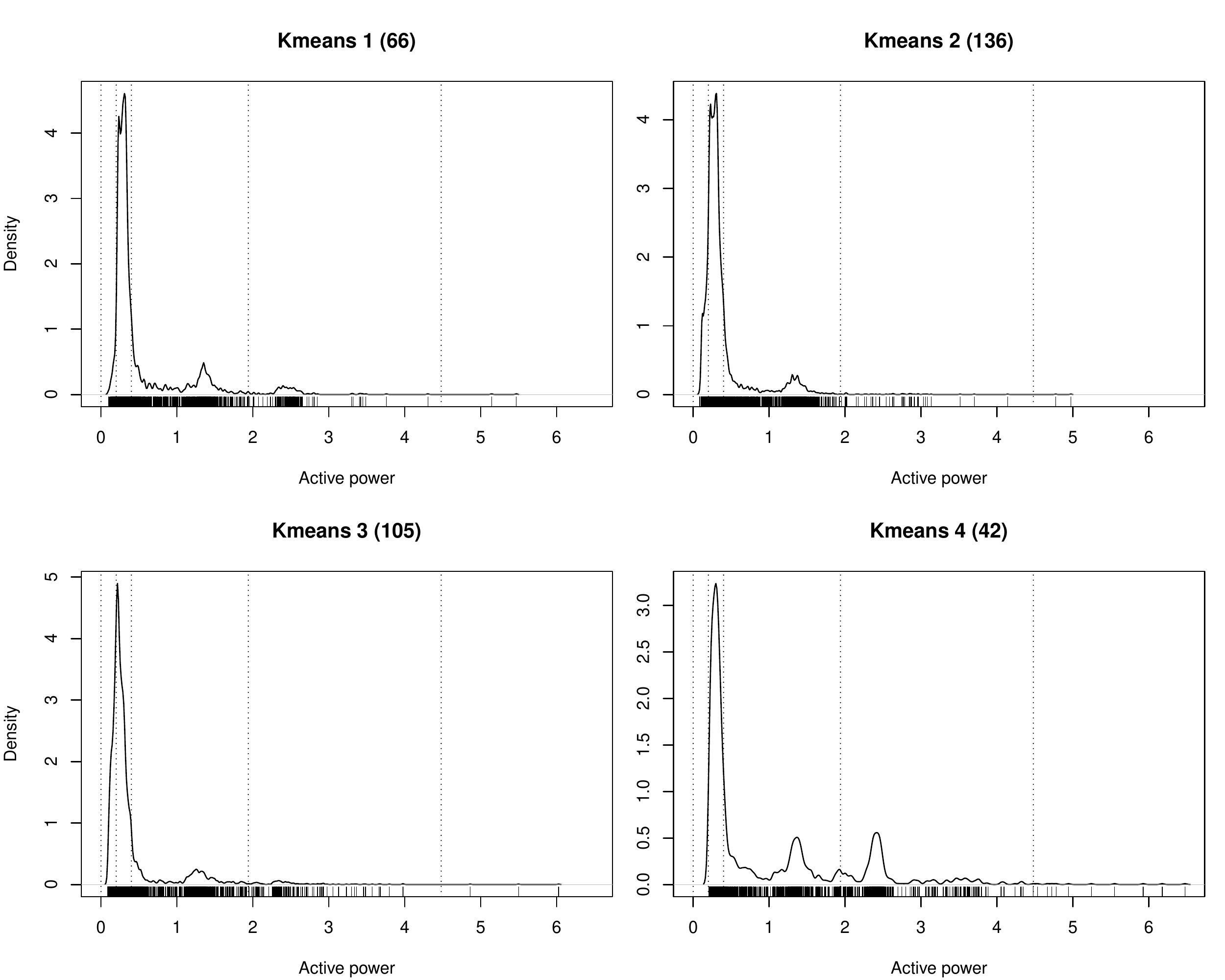}
	\caption{Kernel density estimation of the power consumption measurements between midnight and 6.35 AM for each cluster of curves retrieved using the approach of \cite{HebrailEtAl10}.}
	\label{fig:Hebrkde}
\end{figure}

The density functions for the power consumption are similar for all the four clusters retrieved by the approach of \cite{HebrailEtAl10} during the \textit{night}: for all the four clusters, we observe that the power measurements are very dense around one unique low consumption value that corresponds to the year average power consumption of the studied time segment. As for MODL, the density functions are very similar for the clusters $1$ and $3$ and also very similar to the ones displayed in Figure \ref{fig:Hebrkde}. However, the cluster $4$ is different in that the density peak has been translated to an upper power interval. Finally, the cluster $2$ highlights multimodalities with three power values around which the measurements are dense. This complex pattern has been retrieved by MODL since it based on joint density estimation ; the competing approach cannot track such patterns.

The curves of Figures \ref{fig:MODL4} and \ref{fig:Hebr4} do not clearly highlight the differences between the results. Displaying the calendar with different colors for the 4 clusters gives a more powerful reading of the differences between the results obtained using both methods. This is displayed in Figures \ref{Fig:calendarMODL} and \ref{Fig:calendarHebr}.

The calendar of the clusters retrieved using MODL (see Figure \ref{Fig:calendarMODL}) emphasizes a certain seasonality. Indeed, the way the curves are grouped highlights a link with the weather and the temperatures in France this year. The summer, from June to September, is a season when the temperatures are usually high. On the calendar, there are two clusters corresponding to this period. The rest of the year, the temperatures are lower and lead to an increase of the power consumption which is materialized by the two other clusters. It appears that in late April and early May, the temperature was exceptionally high this year: these days have been classified into the summer clusters. Interestingly, the cluster shown in Figure \ref{Fig:57clusters}.(a) where nobody was at home and the power consumption is low, has been included into a summer cluster (periods from the $23^{th}$ of February to the $2^{nd}$ of March and from the $29^{th}$ of October to the $3^{rd}$ of November).

\begin{figure}[!ht]%
\includegraphics[width=\textwidth]{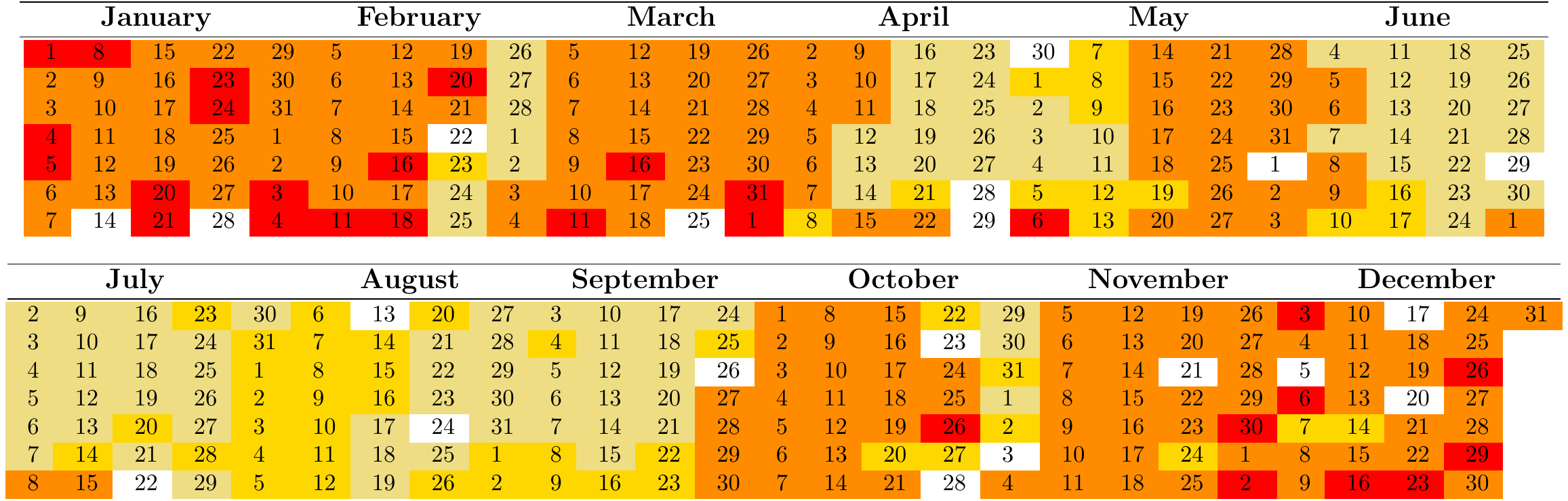}%
\caption{Calendar of the year 2007 retrieved using MODL. Each line represents a day of the week. There are four colors (one per cluster), the redder the color, the higher the average power consumption of the cluster is. The white days correspond to days with missing data.}%
\label{Fig:calendarMODL}%
\end{figure} 

\begin{figure}[!ht]%
\includegraphics[width=\textwidth]{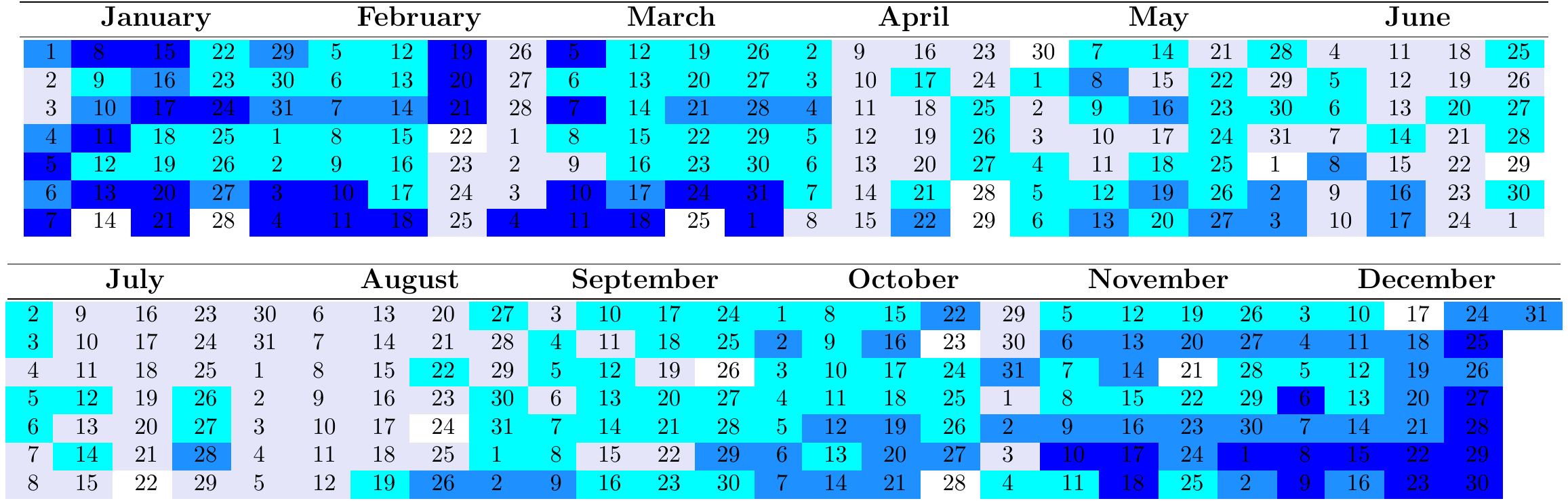}%
\caption{Calendar of the year 2007 retrieved using the approach of \cite{HebrailEtAl10}. Each line represents a day of the week. There are four colors (one per cluster), the bluer the color, the higher the average power consumption of the cluster is. The white days correspond to days with missing data.}%
\label{Fig:calendarHebr}%
\end{figure} 

For its part, the calendar obtained using the approach of \cite{HebrailEtAl10} does not show a seasonality as the one retrieved using MODL does. The clusters are more distributed all over the year. The dark blue cluster (i.e the one with the higher average power consumption) groups however only cold winter days and can be compared to the reddest cluster of the Figure \ref{Fig:calendarMODL}. The palest cluster (i.e the one with the lower average power consumption) characterizes also the warmer days and the days where there is nobody at home (see Figure \ref{Fig:57clusters}.(a)). As for the other ones with intermediate average power consumption, they do not show any correlation with the period of the day and thus do not allow an immediate interpretation.

All in all, both approaches track different patterns and consequently retrieve different clustering schemes. On the one hand, MODL requires no user-defined parameters and is suitable when there are no prior knowledges of the data. Moreover, the approach is supplemented by powerful exploratory analysis tools allowing a global interpretation of the results at different granularity levels. On the other hand, the approach of \cite{HebrailEtAl10} enables a thorough understanding of the clusters by making a time decomposition locally to every cluster. In this practical case study, it appears that both methods are complementary.
\newpage
\section{Conclusion}
\label{sec:Conclusion}
In this paper, we have focused on functional data exploratory analysis, more particularly on curves clustering. The method that is proposed in this paper does not consider the data set as a collection of curves but rather as a set of data points with three features, two continuous, the point coordinates, and one categorical, the curve identifier. By clustering the curves and discretizing each point variable while selecting the best model according to a Bayesian approach, the method behaves as a nonparametric estimator of the joint density of both the curve and point variables. In case of large data sets, the best model tends to be too fine grained for an easy interpretation. To overcome this issue, a post-processing technique is proposed. This technique aims at merging successively the clusters until obtaining a simplified clustering while losing the least accuracy. This process is equivalent to making a hierarchical agglomerative classification, whose dissimilarity measure is a weighted sum of Kullback-Leibler divergences from the new cluster to the two merged clusters. Experimentations have been conducted on an artificial data set in order to highlight interesting properties of the method and on a real world data set, the power consumption of a home over a year. On the one hand, the finest model highlights interesting phenomena such as multimodal distributions for some time segments among the same cluster. As for the post-processing technique, a well-balanced dendrogram and a concave Pareto chart emphasize the ability of the finest model to be simplified with few information loss, leading to a more interpretable clustering. An interpretation of these results has been made focusing on the differences with an alternative approach.

Beyond clustering of curves, the proposed method is able to cluster a collection of distributions. In future works, we plan to extend the method to multidimensional distributions by considering more than two point dimensions.

 \bibliographystyle{apalike}
 \bibliography{author.AKDM} 
 \end{document}